\documentclass{article}

\usepackage{url}
\usepackage{fancybox}
\usepackage{graphicx}
\usepackage{comment}

\usepackage{natbib}
\usepackage{color}
\definecolor{MyBrown}{rgb}{0.3,0,0}
\definecolor{MyBlue}{rgb}{0,0,1}
\definecolor{MyRed}{rgb}{0.5,0,0}
\definecolor{MyGreen}{rgb}{0,0.4,0}
\usepackage[
bookmarks=false,%
bookmarksnumbered=true,%
colorlinks=true,%
linkcolor=MyBlue,%
citecolor=MyBlue,%
filecolor=MyBlue,%
urlcolor=MyGreen%
]{hyperref}

\usepackage{algorithm}
\usepackage{algorithmic}
\usepackage{graphicx}  
\usepackage{dcolumn}   
\usepackage{bm}        
\usepackage{amsthm}
\usepackage{amsmath}
\usepackage{amssymb}
\usepackage{amsfonts}
\usepackage{amscd}
\usepackage{ascmac}
\usepackage{enumerate}
\usepackage{bm}
\usepackage{latexsym}
\usepackage{color}
\usepackage{soul}

\def\QED{\mbox{\rule[0pt]{1.5ex}{1.5ex}}}
\def\endproof{\hspace*{\fill}~\QED\par\endtrivlist\unskip}

\newtheorem{thm} {Theorem}
\newtheorem{lem}[thm] {Lemma}
\newtheorem{df}[thm]{Definition}

\newtheorem{ass}{Assumption}

\def\B{{\mathbb{B}}}
\def\S{{\mathbb{S}}}
\def\R{{\mathbb{R}}}

\def\calA{\mathcal{A}}

\def\calG{\mathcal{G}}

\def\calL{\mathcal{L}}
\def\calR{\mathcal{R}}

\def\bbE{\mathbb{E}}
\def\-{{\mathchar`-}}

\def\argmin{{\rm argmin}}

\def\endproof{\hspace*{\fill}~\QED\par\endtrivlist\unskip}

\def\Label#1{\label{#1}\ [\ #1\ ]\ }
\def\Label{\label}

\definecolor{refkey}{rgb}{0.9451,0.2706,0.4941}
\definecolor{labelkey}{rgb}{0.9451,0.2706,0.4941}

\usepackage[final,nonatbib]{nips_2017}

\usepackage[utf8]{inputenc} 
\usepackage[T1]{fontenc}    
\usepackage{hyperref}       
\usepackage{url}            
\usepackage{booktabs}       
\usepackage{amsfonts}       
\usepackage{nicefrac}       
\usepackage{microtype}      

\title{
Regret Analysis for Continuous Dueling Bandit
}

\author{
  Wataru Kumagai
    \\
    Center for Advanced Intelligence Project\\
    RIKEN\\
    1-4-1, Nihonbashi, Chuo, Tokyo 103-0027, Japan\\  
  \texttt{wataru.kumagai@riken.jp} \\
}

\begin{document}

\maketitle

\begin{abstract}
The dueling bandit is a learning framework wherein the feedback information in the learning process is restricted to a noisy comparison between a pair of actions.
In this research, we address a dueling bandit problem based on a cost function over a continuous space.  
We propose a stochastic mirror descent algorithm 
and show that the algorithm achieves an $O(\sqrt{T\log T})$-regret bound under strong convexity and smoothness assumptions for the cost function.
Subsequently, we clarify the equivalence between regret minimization in dueling bandit and convex optimization for the cost function. 
Moreover, when considering a lower bound in convex optimization, 
our algorithm is shown to achieve the optimal convergence rate in convex optimization and the optimal regret in dueling bandit except for a logarithmic factor.

\end{abstract}

\section{Introduction}\label{sec:Int}

Information systems and computer algorithms often have many parameters which should be tuned.
When cost or utility are explicitly given as numerical values or concrete functions,
the system parameters can be appropriately determined depending on the the values or the functions.
However, 
in a human-computer interaction system,
it is difficult or impossible for users of the system to provide  user preference as numerical values or concrete functions.
{\it Dueling bandit} 
is introduced to model such situations in \citet{yue2009interactively}
and enables us to appropriately tune the parameters based only on comparison results on two parameters by the users.
In the learning process of a dueling bandit algorithm,
the algorithm chooses a pair of parameters called actions (or arms) and receives only the corresponding comparison result.
Since dueling bandit algorithms do not require an individual evaluation value for each action,
they can be applied for wider areas that cannot be formulated using the conventional bandit approach.

When action cost (or user utility) implicitly exists,
the comparison between two actions is modeled via a cost (or utility) function, which represents the degree of the cost (or utility),
and 
a link function, which determines the noise in the comparison results.
We refer to such a modeling method as cost-based (or utility-based) approach 
and employ it in this research. 
\citet{yue2009interactively} first introduced the utility-based approach as a model for a dueling bandit problem.

The cost-based dueling bandit relates to function optimization with noisy comparisons \citep{jamieson2012query,matsui2014parallel}
because  in both frameworks an oracle compare two actions and the feedback from the oracle is represented by binary information.
In particular, the same algorithm can be applied to both frameworks.
However,
as different performance measures are applied to the algorithms in function optimization and dueling bandit,
it has not been demonstrated that
an algorithm that works efficiently in one framework will also perform well in the other framework.
This study clarifies relation between function optimization and dueling bandit thorough their regret analysis.

\subsection{Problem Setup}\Label{sec:PSMA}

In the learning process of the dueling bandit problem,  
a learner presents two points, called actions in a space ${\cal A}$, to an oracle and the oracle returns one-bit feedback to the learner based on which action wins (i.e., which action is more preferable for the oracle).
Here, we denote by $a\succ a'$ the event that $a$ wins $a'$ 
and by $P(a\succ a')$ the probability that $a\succ a'$ happens.
In other words,
we assume that the feedback from the oracle follows the following two-valued random variable:
\begin{eqnarray}
F(a,a')
:=\left\{
\begin{array}{lll}
1&w.p.&P(a\succ a')\\
0&w.p.&1-P(a\succ a'),
\end{array}
\right.
\Label{feedback}
\end{eqnarray}
where the probability $P(a\succ a')$ is determined by the oracle.
We refer to this type of feedback as {\it noisy comparison feedback}.
Unlike conventional bandit problems,
the leaner has to make a decision that is based only on the noisy comparison feedback 
and cannot access the individual values of the cost (or utility) function. 
We further assume that each comparison between a pair of actions is independent of other comparisons. 

The learner makes a sequence of decisions based on the noisy comparisons provided by the oracle. 
After receiving $F(a_t,a'_t)$ at time $t$,
the learner chooses the next two action $(a_{t+1},a'_{t+1})$. 
As a performance measure for an action $a$,
we introduce the minimum win probability:
\begin{eqnarray*}
P^{\ast}(a)=\inf_{a'\in\calA}P(a\succ a').
\Label{prob1}
\end{eqnarray*}
We next quantify the performance of the algorithm using the expected regret as follows:\footnote{
Although the regret in (\ref{regret-DB}) appears superficially different from that in \citet{yue2009interactively},
two regrets can be shown to coincide with each other under Assumptions \ref{ass1}-\ref{ass3} in Subsection \ref{sec:MA}.
}
\begin{eqnarray}
Reg_T^{DB}
=\sup_{a\in{\cal A}}\bbE\left[\sum_{t=1}^{T}\left\{
(P^{\ast}(a) - P^{\ast}(a_t)) + (P^{\ast}(a) - P^{\ast}(a'_t))\right\}\right].
\Label{regret-DB}
\end{eqnarray}

\subsection{Modeling Assumption}\Label{sec:MA}

In this section, we clarify some of the notations and assumptions.
Let an action space ${\cal A}\subset\R^d$ be  compact convex set with non-empty interior.
We denote the Euclidean norm by $\|\cdot\|$.

\begin{ass}\Label{ass1}
There exist functions $f:{\cal A}\to\R$ 
and 
$\sigma:\R\to[0,1]$
such that
 the probability in noisy comparison feedback can be represented  as follows:
\begin{eqnarray}
P(a\succ a') 
= \sigma(f(a') - f(a)).
\Label{eq:nc}
\end{eqnarray}
\end{ass}
In the following, 
we call $f$ in Assumption \ref{ass1}  a cost function and  and $\sigma$ a link function.
Here, the cost function and the link function are fixed for each query to the oracle.
In this sense, our setting is different from online optimization where the objective function changes. 

\begin{df}(Strong Convexity)
A function $f:\R^d\to\R$ is $\alpha$-strongly convex over the set $\calA\subset\R^d$ if for all $x,y\in\calA$ it holds that
\begin{eqnarray*}
f(y)\ge f(x) + \nabla f(x)^{\top}(y-x) + \frac{\alpha}{2}\|y-x\|^2.
\end{eqnarray*}
\end{df}

\begin{df}(Smoothness)
A function $f:\R^d\to\R$ is $\beta$-smooth over the set $\calA\subset\R^d$ if for all $x,y\in\calA$ it holds that
\begin{eqnarray*}
f(y)\le f(x) + \nabla f(x)^{\top}(y-x) + \frac{\beta}{2}\|y-x\|^2.
\end{eqnarray*}
\end{df}

\begin{ass}\Label{ass2}
The cost function $f:\calA\to\R$ is twice continuously differentiable, $L$-Lipschitz, $\alpha$-strongly convex and $\beta$-smooth with respect to the Euclidean norm. 
\end{ass}

From Assumption \ref{ass2}, 
there exists a unique minimizer $a^{\ast}$ of the cost function $f$ since $f$ is strictly convex. 
We set  $B:=\sup_{a,a'\in\calA} f(a')-f(a)$.

\begin{ass}\Label{ass3}
The link function $\sigma:\R\to[0,1]$
is three times 
differentiable and 
rotation-symmetric (i.e., $\sigma(-x) = 1-\sigma(x)$).
Its first derivative is 
positive
and monotonically non-increasing  on $[0,B]$.
\end{ass}

For examples,
the standard logistic distribution function, 
the cumulative standard Gaussian distribution function
and the linear function $\sigma(x)=(1+x)/2$
can be taken to be link functions that satisfy Assumption \ref{ass3}.
We note that link functions often behave like cumulative probability distribution functions.
This is because the sign of the difference between two noisy function values can be regarded as the feedback (\ref{feedback}) which satisfies Assumption \ref{ass1}, and then, the link function $\sigma$ coincides with the cumulative probability distribution function of the noise (see Section $2$ of \citet{jamieson2012query} for more details).
We will discuss the relation of noisy comparison feedback to noisy function values in Section \ref{sec:LB}.

\vspace{-1em}
\subsection{Related Work and Our Contributions}\Label{sec:RW}

Dueling bandit on the continuous action space relates with various optimization methods.
We summarize related studies in the following. 

{\bf  Dueling bandit problem:}
\citet{yue2009interactively} formulated information retrieval systems as a dueling bandit problem.
They reduced this to a problem of optimizing an ``almost"-concave function
and presented a stochastic gradient ascent algorithm based on one-point bandit feedback. 
Subsequently, they showed that their algorithm achieves an $O(T^{3/4})$-regret bound under the differentiability and the strict concavity for a utility function.
\citet{ailon2014reducing} presented reduction methods from dueling bandit  to the conventional bandit under the strong restriction that the link function is linear 
and showed that their algorithm achieves an $O(\sqrt{T\log^3T})$-regret bound. 
We note that dueling bandit has a number of other formulations  \citep{yue2011beat, yue2012k, busa2013top, busa2014preference,  
urvoy2013generic, zoghi2014relative, jamieson2015sparse}.


{\bf  Optimization with one-point bandit feedback:}
In conventional bandit settings,
various convex optimization methods have been studied.
%
\citet{flaxman2005online} 
showed that the gradient of smoothed version of a convex function can be estimated from a one-point bandit feedback and proposed a stochastic gradient descent algorithm which achieves an $O(T^{3/4})$ regret bound under the Lipschitzness condition.
Moreover, assuming the strong convexity and the smoothness for the convex function such as (\ref{ass2}),
\citet{hazan2014bandit} proposed a stochastic mirror descent algorithm which achieves an $O(\sqrt{T\log T})$ regret bound and showed that the algorithm is near optimal because
the upper bound matched the lower bound of $\Omega(\sqrt{T})$ derived by \citet{shamir2013complexity} up to a logarithmic factor in bandit convex optimization.
%

{\bf  Optimization with two-point bandit feedback:}
%
Dueling bandit algorithms require two actions at each round in common with two-point bandit optimization.
In the context of online optimization,
\citet{agarwal2010optimal} first considered convex optimization with two-point feedback.
They proposed a gradient descent-based algorithm and showed that the algorithm achieves the regret bounds of  under the Lipschitzness condition and $O(\log T)$ under the strong convexity condition.
In stochastic convex optimization,
\citet{duchi2015optimal} showed that a stochastic mirror descent algorithm achieves an $O(\sqrt{T})$ regret bound under the Lipschitzness (or the smoothness) condition 
and proved the upper bound to be optimal deriving a matching lower bound $\Omega(\sqrt{T})$. 
Moreover,
in both of online and stochastic convex optimization, 
\citet{shamir2017optimal} showed that a gradient descent-based algorithm achieves an $O(\sqrt{T})$ regret bound with optimal dependence on the dimension under the Lipschitzness condition.
%
However, those two-point bandit algorithms strongly depend on the availability of the difference of function values 
and cannot be directly applied to the case of dueling bandit where
the difference of function values is compressed to one bit in noisy comparison feedback.


{\bf  Optimization with noisy comparison feedback:}
The cost-based dueling bandit relates to function optimization with noisy comparisons \citep{jamieson2012query,matsui2014parallel}
because in both frameworks, the feedback is represented by preference information.
\citet{jamieson2012query} proposed a coordinate descent algorithm 
and proved that the convergence rate of the algorithm achieved an optimal order.\footnote{The optimal order changes depending on the model parameter $\kappa\ge 1$ of the pairwise comparison oracle in \citet{jamieson2012query}.}
\citet{matsui2014parallel} proposed a Newton method-based algorithm 
and proved that its convergence rate was almost equivalent to that of \citet{jamieson2012query}.
They further showed that their algorithm could easily be  parallelized 
and performed better numerically than the dueling bandit algorithm in \citet{yue2009interactively}.
However, since they considered only the unconstrained case in which $\calA=\R^d$,
it is not possible to apply
their algorithm to the setting considered here,
in which the action space is compact.

{\bf  Optimization with one-bit feedback:}
The optimization method of the dueling bandit algorithm is based on one-bit feedback.
In related work,
\citet{zhang2016online} considered stochastic optimization under one-bit feedback.
However, since their approach was restricted to the problem of linear optimization with feedback generated by the logit model,
it cannot be applied to the problem addressed in the current study. 

 
{\bf  Our contributions:}
In this paper,
we consider the cost-based dueling bandit under Assumptions \ref{ass1}-\ref{ass3}.
While the formulation is similar to that of  \citet{yue2009interactively}, 
Assumptions \ref{ass2} and \ref{ass3} are stronger than those used in that study.
On the other hand, 
we impose the weaker assumption on the link function than that of \citet{ailon2014reducing}.
%
\citet{yue2009interactively} showed that a stochastic gradient descent algorithm can be applied to dueling bandit.
Thus, it is naturally expected that a stochastic mirror descent algorithm, which achieves the (near) optimal order in convex optimization with one/two-point bandit feedback, can be applied to dueling bandit setting and achieves good performance.
Following this intuition, we propose a mirror descent-based algorithm.
Our key contributions can be summarized as follows: 
\setlength{\leftmargini}{23pt} 
\begin{itemize}
\item We propose a stochastic mirror descent algorithm with noisy comparison feedback. 
\item We provide an $O(\sqrt{T\log T})$-regret bound for our algorithm in dueling bandit.
\item We clarify the relation between the cost-based dueling bandit and convex optimization in terms of their regrets
and show that our algorithm can be applied to convex optimization.
\item We show that the convergence rate of our algorithm is $O(\sqrt{\log T/T})$ in convex optimization.
\item We derive a lower bound in convex optimization with noisy comparison feedback and 
show  
our algorithm to be near optimal in both dueling bandit and convex optimization.    
\end{itemize}

\section{Algorithm and Main Result}\Label{sec:AMR}

\subsection{Stochastic Mirror Descent Algorithm}

We first prepare the notion of a self-concordant function on which our algorithm is constructed (see e.g., \citet{nesterov1994interior}, Appendix F in \citet{griva2009linear}).
\begin{df}
A function $\calR:{\rm int}({\cal A})\to\R$ is considered self-concordant if the following two conditions hold:
\begin{enumerate}
\item  $\calR$ is three times continuously differentiable and convex, and approaches infinity along any sequence of points approaching the boundary of ${\rm int}({\cal A})$.
\item  For every $h\in\R^d$ and $x\in int({\cal A})$,
$|\nabla^3 \calR(x)[h, h, h]| \le 2(h^{\top}\nabla^2\calR(x)h)^{\frac{3}{2}}$ holds,
where  
$\nabla^3\calR(x)[h,h,h]:= \frac{\partial^3\calR}{\partial t_1 \partial t_2 \partial t_3}(x+t_1h+t_2h+t_3h)\big|_{t_1=t_2=t_3=0}.$
\end{enumerate}
In addition to these two conditions,
if $\calR$ satisfies the following condition for a positive real number $\nu$, 
it is called a $\nu$-self-concordant function:
\begin{enumerate}
\item[3.] For every $h\in\R^d$ and $x\in int({\cal A})$, $|\nabla \calR(x)^{\top}h| \le \nu^{\frac{1}{2}} (h^{\top}\nabla^2\calR(x)h)^{\frac{1}{2}}$.
\end{enumerate}
\end{df}
In this paper,
we assume the Hessian $\nabla^2\calR(a)$ of a $\nu$-self-concordant function to be full-rank over $\calA$
and 
$\nabla \calR(\hbox{int}(\calA))=\R^d$.
\citet{bubeck2014entropic} showed that such a $\nu$-self-concordant function satisfying $\nu=(1+o(1))d$ will always exist as long as the dimension $d$ is sufficiently large.
We next propose Algorithm \ref{A1}, which we call {\it NC-SMD}.
This can be regarded as {\it stochastic mirror descent} with noisy comparison feedback.

We make three remarks on Algorithm \ref{A1}.
First, 
the function $\calR_t$ is self-concordant though not $\nu$-self-concordant.
The second remark is as follows.
Let us denote the local norms by $\|a\|_{w}=\sqrt{ a^{\top} \nabla^2\calR(w)a}$.
Then, if $\calR$ is a self-concordant function for $\calA$,  
the Dikin Ellipsoid
$\{a'\in\calA|~\|a'-a\|_{a}\le1\}$
centered at $a$ is entirely contained in ${\rm int}(\calA)$  for any $a\in {\rm int}(\calA)$.
Thus,
 $a'_t:=a_t+\nabla^2{\calR_t}(a_t)^{-\frac{1}{2}} u_t$ in Algorithm \ref{A1} is contained in ${\rm int}(\calA)$ for any $a_t\in {\rm int}(\calA)$ and a unit vector $u_t$.
This shows a comparison between actions $a_t$ and  $a'_t$ to be feasible.
Our third remark is as follows.
Since the self-concordant function $\calR_t$ at round $t$ depends on the past actions $\{a_i\}_{i=1}^t$,
it may be thought that those past actions are stored during the learning process.
However, note that only $\nabla \calR_t$ and $\nabla^2 \calR_t$ are used in the algorithm; 
$\nabla \calR_t$ depends only on $\sum_{i=1}^t a_t$ and $\nabla^2 \calR_t$ does not depend on the past actions.
Thus, only the sum of past actions must be stored, rather than all past actions.

\begin{algorithm}[t]
   \caption{Noisy Comparison-based Stochastic Mirror Descent (NC-SMD)}
   \label{A1}
\begin{algorithmic}
	\STATE {\bf Input:} Learning rate $\eta $, $\nu$-self-concordant function $\calR$, time horizon $T$, tuning parameters $\lambda, \mu$
	\STATE {\bf Initialize:} $a_1=\argmin_{a\in\calA} {\calR}(a)$.
 	\FOR{$t=1$ to $T$}
    		\STATE Update $\calR_t(a) = \calR(a)+\frac{\lambda\eta}{2}\sum_{i=1}^{t}\|a-a_i\|^2 + \mu\|a\|^2$
		\STATE Pick a unit vector $u_t$ uniformly at random
		\STATE Compare $a_t$ and $a'_t:=a_t+\nabla^2{\calR_t}(a_t)^{-\frac{1}{2}} u_t$ and receive $F(a'_t, a_t)$
		\STATE Set $\hat{g}_t=F(a'_t,a_t)d\nabla^2{\calR_t}(a_t)^{\frac{1}{2}}u_t$
		\STATE Set $a_{t+1}=\nabla {\calR_t^{-1}}(\nabla {\calR_t}(a_t) - \eta \hat{g}_t)$
	\ENDFOR
	\STATE {\bf Output:} $a_{T+1}$ 
\end{algorithmic}
\end{algorithm}

\subsection{Main Result: Regret Bound}\Label{sec:RA}

From Assumption \ref{ass2} and the compactness of $\calA$,
the diameter $R$ and
$B:=\sup_{a,a'\in\calA} f(a')-f(a)$ are finite.
From Assumption \ref{ass3}, 
there are exist positive constants $l_0$, $L_0$, $B_2$ and $L_2$ such that 
the first derivative $\sigma'$ of the link function is bounded as $l_0\le\sigma'\le L_0$ on $[-B, B]$ 
and the second derivative $\sigma''$ is bounded above by $B_2$ 
and 
$L_2$-Lipschitz on $[-B, B]$. 
We use the constants below.

The following theorem shows that with appropriate parameters,
NC-SMD (Algorithms \ref{A1}) achieves an $O(\sqrt{T\log T})$-regret bound.
\begin{thm} \Label{thm2}
We set $C:=\nu + \frac{ B_2L^2 +(L+1)L_0\beta}{2\lambda}$.
When the tuning parameters satisfy 
$\lambda\le l_0\alpha/2$, 
$\mu \ge \left(L_0^3 L_2/\lambda\right)^2$ 
and the total number $T$ of rounds satisfies $T\ge C\log T$.
Algorithm \ref{A1} with a $\nu$-self-concordant function and the learning parameter $\eta=\frac{1}{2d}\sqrt{\frac{C\log T }{T}}$
 achieves the following regret bound under Assumptions \ref{ass1}-\ref{ass3}:
\begin{eqnarray}
Reg_T^{DB}
&\le&4d\sqrt{C T\log T}+2LL_0R.
\Label{reg-bound}
\end{eqnarray}
\end{thm}

\section{Regret Analysis}\Label{sec:RA}

We prove Theorem \ref{thm2} in this section.
The proofs of lemmas in this section are provided in supplementary material.

\subsection{Reduction to Locally-Convex Optimization}\Label{sec:RLCO}

We first reduce the dueling bandit problem to a locally-convex optimization problem.
We define  $P_b(a):=\sigma(f(a)-f(b))$ for $a,b\in\calA$ and $P_t(a):=P_{a_t}(a)$.
For a cost function $f$ and a self-concordant function $\calR$,
we set $a^*:=\argmin_{a\in\calA}f(a)$, $a_1:=\argmin_{a\in\calA}\calR(a)$ and
$a^{\ast}_T:=\frac{1}{T}a_1 + (1-\frac{1}{T})a^{\ast}$.
The regret of dueling bandit is bounded as follows.
\begin{lem}\Label{upper-bound}
The regret of Algorithm \ref{A1} is bounded as follows:
\begin{eqnarray}
Reg_T^{DB}
&\le& 2\bbE\left[\sum_{t=1}^{T}(P_t(a_t) - P_t(a^{\ast}_T))\right]
+\frac{LL_0\beta}{\lambda\eta}\log T
+2LL_0R.
\Label{upper1}
\end{eqnarray}
\end{lem}

The following lemma shows that $P_b$ inherits the smoothness of $f$ globally.
\begin{lem}\Label{smoothness}
The function $P_b$ is $(L_0 \beta+B_2L^2)$-smooth for an arbitrary $b\in\calA$.
\end{lem}

Let $\B$ be the unit Euclidean ball, $\B(a,\delta)$ the ball centered at $a$ with radius $\delta$ and
${\cal L}(a,b)$ the line segment between $a$ and $b$.
In addition,
for  $a,b\in\calA$,
let ${\cal A}_\delta(a,b):=\cup_{a'\in{\calL}(a,b)}\B(a',\delta)\cap\calA$. 
The following lemma guarantees the local strong convexity of $P_b$. 

\begin{lem}\Label{convexity}
The function $P_b$ is $\frac{1}{2}l_0 \alpha$-strongly convex on $\calA_\delta(a^{\ast},b)$
when $\delta\le\frac{l_0 \alpha}{2L_0^3 L_2}$.
\end{lem}

\subsection{Gradient Estimation}\Label{sec:GE}

We note that $a_t+\nabla^2{\calR_t}(a_t)^{-\frac{1}{2}} x$ for $x\in\B$ is included in $\calA$ due to the properties of the Dikin ellipsoid.
We introduce the smoothed version of $P_t$ over ${\rm int}({\cal A})$:
\begin{eqnarray}
\hat{P}_t(a)
&:=&\bbE_{x\in \B}[P_t(a+\nabla^2{\calR_t}(a_t)^{-\frac{1}{2}} x)]\\
&=&\bbE_{x\in \B}[\sigma(f(a+\nabla^2{\calR_t}(a_t)^{-\frac{1}{2}} x)-f(a_t))].
\Label{hat}
\end{eqnarray}
Next, we adopt the following estimator for the gradient of $\hat{P}_t$:
\begin{eqnarray*}
\hat{g}_t
:=
F(a_t+\nabla^2{\calR_t}(a_t)^{-\frac{1}{2}} u_t,a_t)d \nabla^2{\calR_t}(a_t)^{\frac{1}{2}}u_t,
\end{eqnarray*}
where $u_t$ is drawn from the unit surface $\S$ uniformly.
We then derive the unbiasedness of $\hat{g}_t$ as follows.
\begin{lem}\Label{unbiased}
\begin{eqnarray*}
\bbE[\hat{g}_t|a_t]
=\nabla \hat{P}_t(a_t).
\end{eqnarray*}
\end{lem}

\subsection{Regret Bound with Bregman Divergence}\Label{sec:RB}

From Lemma \ref{upper-bound}, 
the regret analysis in dueling bandit is reduced to the minimization problem of the regret-like value of $P_t$.
Since $P_t$ is  globally smooth and locally strongly convex from Lemmas \ref{smoothness} and \ref{convexity},
we can employ convex-optimization methods.
Moreover, since $\hat{g}_t$ is an unbiased estimator for the gradient of the smoothed version of $P_t$ from Lemma \ref{unbiased},
it is expected that stochastic mirror descent (Algorithm \ref{A1}) with $\hat{g}_t$ is effective to the minimization problem.
In the following, making use of the property of stochastic mirror descent,
we bound the regret-like value of $P_t$ by the Bregman divergence, 
and subsequently prove Theorem \ref{thm2}. 

\begin{df}
Let $\calR$ be a continuously differentiable strictly convex function  on ${\rm int}(\calA)$.
Then, the Bregman divergence  associated with $\calR$ is defined by
\begin{eqnarray*}
D_{\calR}(a,b)
=\calR(a)-\calR(b)-\nabla\calR(b)^{\top}(a-b).
\end{eqnarray*}
\end{df}

\begin{lem}\Label{DB-bound}
When $\lambda\le l_0\alpha/2$ and
$\mu \ge \left(L_0^3 L_2/\lambda\right)^2$,
the regret of Algorithm \ref{A1} is bounded for any $a\in{\rm int}(\calA)$ as follows:
\begin{eqnarray}
&&\hspace{-5mm}\bbE\left[\sum_{t=1}^{T}(P_t(a_t) - P_t(a))\right] \nonumber\\
&\hspace{-8mm}\le&\hspace{-5mm}\frac{1}{ \eta}\hspace{0.1em}\left({\calR}(a) - {\calR}(a_1) + \bbE\left[\sum_{t=1}^T D_{\calR_t^{\ast}}(\nabla{\calR}(a_t)-\eta \hat{g}_t,\nabla{\calR}(a_t))\right]\right)
+\frac{L_0\beta + B_2L^2}{\lambda\eta}\log T,
\Label{upper}
\end{eqnarray}
where $\calR_t^{\ast}(a):=\sup_{x\in\R^d} \langle x,a \rangle - \calR_t(a)$ is the Fenchel dual of $\calR_t$.
\end{lem}

The Bregman divergence in Lemma \ref{DB-bound} is bounded as follows.
\begin{lem} \Label{div-bound}
When $\eta\le\frac{1}{2d}$,
the sequence $a_t$ output by Algorithm \ref{A1} satisfies 
\begin{eqnarray}
D_{\calR_t^{\ast}}(\nabla{\calR_t}(a_t)-\eta \hat{g}_t,\nabla{\calR_t}(a_t))
\le 4d^2\eta^2.
\Label{eq:div-bound}
\end{eqnarray}
\end{lem}

\vspace{1em}
\noindent{\bf [Proof of Theorem \ref{thm2}]} 
From Lemma $4$ of \citet{hazan2014bandit},
the $\nu$-self-concordant function $\calR$ satisfies
\begin{eqnarray*}
\calR(a^{\ast}_T) - \calR(a_1)
\le \nu\log \frac{1}{1-\pi_{a_1}(a^{\ast}_T)},
\end{eqnarray*}
where $\pi_{a}(a'):=\inf\{r\ge 0| a + r^{-1}(a'-a)\}$ is the Minkowsky function. 
Since $\pi_{a_1}(a^{\ast}_T) \le 1-T^{-1}$ from the definition of $a^{\ast}_T$, 
we obtain 
\begin{eqnarray*}
\calR(a^{\ast}_T) - \calR(a_1)
\le \nu\log T.
\end{eqnarray*}
Note that the condition $\eta\le \frac{1}{2d}$ in Lemma \ref{div-bound} is satisfied due to $T\ge C\log T$.
Combining Lemmas \ref{upper-bound}, \ref{DB-bound} and \ref{div-bound}, 
we have 
\begin{eqnarray*}
Reg_T^{DB}
&\le&\frac{2}{ \eta}\hspace{0.1em}\left(\nu\log T + 4d^2\eta^2T\right)
+\frac{L_0\beta + D_{\sigma''}L^2}{\lambda\eta}\log T
+\frac{LL_0\beta}{l_0\alpha\eta}
+2LL_0R\\
&\le&\left(2\nu + \frac{B_2L^2 +(L+1)L_0\beta}{\lambda}\right)\frac{\log T}{\eta}
+ 8d^2T\eta+2LL_0R.
\end{eqnarray*}
Thus, when $\eta$ is defined in Theorem \ref{thm2},
the regret bound (\ref{reg-bound}) is obtained.
\endproof

\section{Convergence Rate in Convex Optimization}\Label{sec:CDBP}

In the previous sections,
we considered the minimization problem for the regret of dueling bandit.
In this section,
as an application of the approach,
we show that the averaged action of NC-SMD (Algorithm \ref{A1}) minimize the cost function $f$ in (\ref{eq:nc}).

To derive the convergence rate of our algorithm,
we introduce the regret of function optimization and establish a connection between the regrets of dueling bandit and function optimization.
In convex optimization with noisy comparison feedback, 
the learner chooses a pair $(a_t, a'_t)$ of actions in the learning process and suffers a loss $f(a_t)+f(a'_t)$.
Then, the regret of the algorithms in function optimization is defined as follows:
\begin{eqnarray}
Reg_T^{FO}
&:=&\bbE\left[\sum_{t=1}^{T}
(f(a_t)-f(a^{\ast}))+(f(a'_t)-f(a^{\ast}))
\right],
\Label{regret-FO}
\end{eqnarray}
where $a^{\ast}=\argmin_{a\in\calA} f$.

Recalling that the positive constants $l_0$ and $L_0$ satisfy $l_0\le\sigma'\le L_0$ on $[-B, B]$,
where  $B:=\sup_{a,a'\in\calA} f(a')-f(a)$,
the regrets of function optimization (\ref{regret-FO}) and dueling bandit (\ref{regret-DB}) are related as follows:
\begin{lem}\Label{DBFO}
\begin{eqnarray}
\frac{Reg_T^{DB}}{L_0}
\le
Reg_T^{FO}
\le
\frac{Reg_T^{DB}}{l_0}.
\Label{reg-ine}
\end{eqnarray}
\end{lem}

Theorem \ref{thm2} and Lemma \ref{DBFO} give an $O(\sqrt{T\log T})$-upper bound of the regret of function optimization in Algorithm \ref{A1} under the same conditions as Theorem \ref{thm2}.
Given the convexity of $f$,
the average of the chosen actions of any dueling bandit algorithm $\bar{a}_T:=\frac{1}{2T}\sum_{t=1}^T(a_t+a'_t)$
satisfies 
\begin{eqnarray}
\bbE[f(\bar{a}_T) - f(a^{\ast})]
\le
\frac{Reg_T^{FO}}{2T}.
\Label{OTB-conv}
\end{eqnarray}
Thus, if an optimization algorithm has a sub-linear regret bound,
the above online-to-batch conversion guarantees  convergence to the optimal point.

\begin{thm}\label{thm:upper}
Under Assumptions \ref{ass1}-\ref{ass3},
the averaged action $\bar{a}_T$  satisfies the following when $T\ge C\log T$:
\begin{eqnarray*}
\bbE[f(\bar{a}_T) - f(a^{\ast})]
\le
\frac{1}{l_0} \left(2d\sqrt{\frac{\nu\log T +C}{T}} + \frac{LL_0R}{T}\right),
\end{eqnarray*}
where $C$ is the constant defined in Theorem \ref{thm2}. 
\end{thm}

Theorem \ref{thm:upper} shows the convergence rate of NC-SMD (Algorithm \ref{A1}) to be $O(d\sqrt{\log T/T})$.


\section{Lower Bound}\Label{sec:LB}

We next derive a lower bound in convex optimization with noisy comparison feedback.
To do so, we employ a lower bound of convex optimization with {\it noisy function feedback}.
In a setting where the function feedback is noisy, 
we query a point $a\in\calA$ and obtain a noisy function value $f(a)+\xi$, 
where $\xi$ is a zero-mean random variable with a finite second moment and independent for each query.
\footnote{In general, the noise $\xi$ can depend on the action $a$. See e.g.  \citet{shamir2013complexity} for more details.}

\begin{thm}\Label{thm:lower}
Assume that the action space $\calA$ is the $d$-dimensional unit Euclidean ball $\B_d$ and that the link function $\sigma_{\calG}$ is the cumulative distribution function of the zero-mean Gaussian random variable with variance $2$.
Let the number of rounds $T$ be fixed. 
Then, for any algorithm with noisy comparison feedback,
there exists a function $f$ over $\B_d$ which is twice continuously differentiable, $0.5$-strongly convex and $3.5$-smooth such that the output $a_T$ of the algorithm satisfies 
\begin{eqnarray}
\bbE[f(a_T) - f(a^*)]
\ge 0.004 \min\left\{1, \frac{d}{\sqrt{2T}}\right\}.
\label{eq:lower}
\end{eqnarray}
\end{thm}

\begin{proof}
The probability distribution of noisy comparison feedback $F(a,a')$ with the link function $\sigma_{\calG}$ can be realized by noisy function feedback with thestandard Gaussian noise as follows.
Two noisy function values $f(a) + \xi$ and $f(a') + \xi'$ can be obtained using the noisy function feedback twice, where $\xi$ and $\xi'$ are independent standard Gaussian random variables.
Then, 
the probability distribution of the following random variable coincide with that of $F(a,a')$ for arbitrary $a,a'\in\calA$: 
\begin{eqnarray}
 {\rm sign}(f(a) + \xi - (f(a') + \xi'))
= {\rm sign}( f(a) - f(a') + (\xi - \xi')). 
\Label{difference}
\end{eqnarray}
Here, note that $\xi - \xi'$ is the zero-mean Gaussian random variable with variance $2$.
Thus, single noisy comparison feedback with the link function $\sigma_{\calG}$  for any actions can be obtained by using noisy function feedback with standard Gaussian noise twice.
This means that if any algorithm with $2T$-times noisy function feedback is unable to achieve a certain performance, 
any algorithm with $T$-times noisy comarison feedback is similarly unable to achieve that performance.
Thus, to derive Theorem \ref{thm:lower},
it is sufficient to show a lower bound of convergence rate with noisy function feedback.
The following lower bound is derived by Theorem $7$ of \citet{shamir2013complexity}. 
\begin{thm}\Label{shamir}\citep{shamir2013complexity}
Let the number of rounds $T$ be fixed. 
Suppose that the noise $\xi$ at each round is a standard Gaussian random variable.
Then, for any algorithm with noisy function feedback, 
there exists a function $f$ over $\B_d$ which is twice continuously differentiable, $0.5$-strongly convex and $3.5$-smooth
such that the output $a_T$ satisfies 
\begin{eqnarray*}
\bbE[f(a_T) - f(a^*)]
\ge 0.004 \min\left\{1, \frac{d}{\sqrt{T}}\right\}.
\end{eqnarray*}
\end{thm}
By the above discussion and from Theorem \ref{shamir},
we obtain Theorem \ref{thm:lower}.
\end{proof}

Combining Theorem \ref{thm:upper} and Theorem \ref{thm:lower},
the convergence rate of NC-SMD (Algorithm \ref{A1}) is near optimal with respect to the number of rounds  $T$. 
In addition, when the parameter $\nu$ of the self-concordant function is of constant order with respect to the dimension $d$ of the space $\calA$,
the convergence rate of NC-SMD is optimal with respect to $d$. 
However, it should be noted that the parameter $\nu$ of a self-concordant function is often of the order of $\Theta(d)$ for compact convex sets including the simplex and the hypercube. 

As a consequece of Lemma \ref{DBFO}, (\ref{OTB-conv}), and Theorems \ref{thm2} and \ref{thm:lower},
the optimal regrets of dueling bandit and function optimization are of the order $\sqrt{T}$ except for the logarithmic factor
and NC-SMD achieves the order.
To the best of our knowledge,
this is the first algorithm with the optimal order 
in the continuous dueling bandit setting with the non-linear link function.

Finally, we provide an interesting observation on convex  optimization.
When noisy function feedback is available,
the optimal regret of function optimization is of the order $\Theta(\sqrt{T})$ under strong convexity and smoothness conditions  \citep{shamir2013complexity}.
However, 
even when noisy function feedback is "compressed" into one-bit information as in (\ref{difference}),
our results show that NC-MSD (Algorithm \ref{A1}) achieves almost the same order $O(\sqrt{T\log T})$ about the regret of function optimization as long as the cumulative probability distribution of the noise satisfies Assumption \ref{ass3}.\footnote{\citet{jamieson2012query} provided a similar observation. However, their upper bound of the regret was derived only when the action space is the whole of Euclidean space (i.e., $\calA=\R^d$) and the assumption for noisy comparison feedback is different from ours (Assumption \ref{ass1}).}

\section{Conclusion}\Label{sec:CD}

We considered a dueling bandit problem over a continuous action space
and proposed a stochastic mirror descent. 
By introducing Assumptions \ref{ass1}-\ref{ass3}, 
we proved that our algorithm achieves an $O(\sqrt{T\log T})$-regret bound.
We further considered convex optimization under noisy comparison feedback
and showed that the regrets of dueling bandit and function optimization are essentially equivalent.
Using the connection between the two regrets, 
it was shown that our algorithm achieves a convergence rate of $O(\sqrt{\log T/T})$ in the framework of function optimization with noisy comparison feedback.
Moreover, we derived a lower bound of the convergence rate in convex optimization
and showed that our algorithm achieves near optimal performance in dueling bandit and convex optimization.
Some open questions still remain. 
While we have only dealt with bounds which hold in expectation,
the derivation of the high-probability bound is a problem that has not been solved.
While the analysis of our algorithm relies on strong convexity and smoothness, 
a regret bound without these conditions is also important.

\newpage

\section*{Acknowledgment}
We would like to thank Professor Takafumi Kanamori for helpful comments.
This work was supported by JSPS KAKENHI Grant Number 17K12653.

\small
\bibliographystyle{econ}
\bibliography{MLpapers}

\newpage

\begin{center}
{\Large Appendix for {\bf Regret Analysis for Continuous Dueling Bandit}} 
\end{center}

\normalsize
\appendix
\section{Appendix: Technical Proofs}

[{\bf Proof of Lemma \ref{upper-bound}}]
From direct calculation, 
{\allowdisplaybreaks
\begin{align*}
Reg_T^{DB}
&=\bbE\left[\sum_{t=1}^{T}\{2 \sigma(f(a_t)-f(a_t)) -\sigma(f(a^{\ast})-f(a_t))-\sigma(f(a^{\ast})-f(a_t+\nabla^2{\calR_t}(a_t)^{-\frac{1}{2}}u_t))\}\right]\nonumber\\
&=2\bbE\left[\sum_{t=1}^{T}\{\sigma(f(a_t)-f(a_t)) -\sigma(f(a^{\ast})-f(a_t))\}\right]\nonumber\\
&~~~+\bbE\left[\sum_{t=1}^{T}\{\sigma(f(a^{\ast})-f(a_t)) - \sigma(f(a^{\ast})-f(a_t+\nabla^2{\calR_t}(a_t)^{-\frac{1}{2}}u_t))\}\right].\\
&=2\bbE\left[\sum_{t=1}^{T}(P_t(a_t) - P_t(a^{\ast}))\right]\nonumber\\
&~~~+\bbE\left[\sum_{t=1}^{T}\{\sigma(f(a^{\ast})-f(a_t)) - \sigma(f(a^{\ast})-f(a_t+\nabla^2{\calR_t}(a_t)^{-\frac{1}{2}}u_t))\}\right].
%
\nonumber
\end{align*}}

Here, we note that $f(a^{\ast})-f(a)\le0$ for any $a\in\calA$ due to the definition of $a^{\ast}$ and that $\sigma$ is convex on $(-\infty,0)$ because the link function is rotation symmetric and its derivative is monotonically non-increasing on positive real numbers from Assumption \ref{ass3}.
Thus, Jensen's inequality derives 
\begin{eqnarray*}
&&\bbE\left[\sigma(f(a^{\ast})-f(a_t+\nabla^2{\calR_t}(a_t)^{-\frac{1}{2}}u_t))|a_t\right]\\
&\ge& \sigma(\bbE[f(a^{\ast})-f(a_t+\nabla^2{\calR_t}(a_t)^{-\frac{1}{2}}u_t)|a_t])\\
&=& \sigma(f(a^{\ast})-\hat{f}(a_t)).
\end{eqnarray*}
In addition, $f(a_t)\le\hat{f}(a_t)$ holds due to the convexity of $f$.
As $\sigma$ is monotonically non-decreasing from Assumption \ref{ass3},
we have
\begin{eqnarray*}
&&\bbE\left[\sum_{t=1}^{T}\{\sigma(f(a^{\ast})-f(a_t)) - \sigma(f(a^{\ast})-f(a_t+\nabla^2{\calR_t}(a_t)^{-\frac{1}{2}}u_t))\}\right]\\
&=&\bbE\left[\bbE\left[\sum_{t=1}^{T}\{\sigma(f(a^{\ast})-f(a_t)) - \sigma(f(a^{\ast})-f(a_t+\nabla^2{\calR_t}(a_t)^{-\frac{1}{2}}u_t))\}|a_t\right]\right]\\
&\le&\bbE\left[\sum_{t=1}^{T}\{\sigma(f(a^{\ast})-f(a_t)) - \sigma(f(a^{\ast})-\hat{f}(a_t))\}\right]\\
&\le& LL_0\bbE\left[\sum_{t=1}^{T}(\hat{f}(a_t)-f(a_t)) \right]\\
&\le& \frac{LL_0\beta}{2}\bbE\left[\sum_{t=1}^{T}\bbE_{x\in\B}[\|\nabla^2{\calR_t}(a_t)^{-\frac{1}{2}} x\|^2] \right]\\
&\le& \frac{LL_0\beta}{2}\bbE\left[\sum_{t=1}^{T}\frac{1}{\lambda\eta t} \right]\\
&\le& \frac{LL_0\beta}{\lambda\eta}\log T,
\end{eqnarray*}
where we used the property of the $\beta$-smoothness of $f$ in the third inequality.
Thus, we obtain (\ref{upper1})
\begin{eqnarray*}
Reg_T^{DB}
&\le& 2\bbE\left[\sum_{t=1}^{T}(P_t(a_t) - P_t(a^{\ast}))\right] 
+  \frac{LL_0\beta}{\lambda\eta}\log T.
\end{eqnarray*}
Here, we have
\begin{eqnarray*}
 \bbE\left[\sum_{t=1}^{T}(P_t(a_t) - P_t(a^{\ast}))\right]
&=&
\bbE\left[\sum_{t=1}^{T}(P_t(a_t) - P_t(a^{\ast}_T))\right]
+\bbE\left[\sum_{t=1}^{T}(P_t(a^{\ast}_T) - P_t(a^{\ast}))\right].
\end{eqnarray*}
From the definition of $a^{\ast}_T$,
we have 
\begin{eqnarray*}
P_t(a^{\ast}_T) - P_t(a^{\ast})
\le LL_0\|a^{\ast}_T - a^{\ast}\|
\le \frac{LL_0R}{T},
\end{eqnarray*}
where $R$ is the diameter of $\calA$.
Thus (\ref{upper1}) is obtained.
\endproof

[{\bf Proof of Lemma \ref{smoothness}}]
From direct calculation,
we obtain that 
\begin{eqnarray}
\nabla P_b(a)
&=&\sigma'(f(a)-f(b)) \nabla f(a),\\
\nabla^2 P_b(a)
&=&\sigma'(f(a)-f(b)) \nabla^2 f(a) +\sigma''(f(a)-f(b))\nabla f(a)\nabla f(a)^{\top}.
\Label{hess1}
\end{eqnarray}
Then, it is sufficient to give upper bounds on the first and second terms in (\ref{hess1}) in the sense of matrix inequalities as
\begin{eqnarray}
\sigma'(f(a)-f(b)) \nabla^2 f(a) 
&\le& L_0 \beta I,
\Label{beta'}\\
\sigma''(f(a)-f(b))\nabla f(a)\nabla f(a)^{\top}
&\le& B_2L^2 I,
\Label{beta''}
\end{eqnarray}
where $I$ is the $d\times d$ identity matrix.
The inequality (\ref{beta'}) follows from the $L_0$-Lipschitzness of $\sigma$ and  the $\beta$-smoothness of $f$.
The inequality (\ref{beta''}) follows from the $B_2$-boundedness of $\sigma''$ and the $L$-Lipschitzness of $f$.
\endproof

[{\bf Proof of Lemma \ref{convexity}}]
We first show that $P_b$ is $l_0 \alpha$-strongly convex on ${\cal L}(a^{\ast},b)$. 
Since (\ref{hess1}) holds for any $a\in{\cal L}(a^{\ast},b)$, 
it is sufficient to give lower bounds on the first and second terms in (\ref{hess1}) for $a$ in ${\cal A}_t$ in the sense of matrix inequalities. 
In the following, we show 
\begin{eqnarray}
\sigma'(f(a)-f(b)) \nabla^2 f(a) 
&\ge& l_0 \alpha I,
\Label{alpha'}\\
\sigma''(f(a)-f(b))\nabla f(a)\nabla f(a)^{\top}
&\ge& 0.
\Label{alpha''}
\end{eqnarray}
Since $l_0\le\sigma'$ and $f$ is $\alpha$-strongly convex, 
we obtian (\ref{alpha'}).
Next, we show (\ref{alpha''}).
Since $\sigma'$ is monotonically non-decreasing on $[-B,0]$, 
$\sigma''(y)$ is negative only if $y$ is positive.
Note that  $f(a)-f(b)\le0$ for any $a\in{\cal L}(a^{\ast},b)$ since $-f(b)+f(a^{\ast})\le0$ and $f$ is convex.
Thus, we have (\ref{alpha''}).

Next, we show that $P_b$ is $\frac{1}{2}l_0 \alpha$-strongly convex on $\calA_\delta(a^{\ast},b)$
when $\delta\le\frac{l_0 \alpha}{4L_0^3 L_2}$.
For an arbitrary $\tilde{a}\in\calA_\delta(a^{\ast},b)$,
there exists $a\in\calL(a^{\ast},b)$ and $y\in\B(0,\delta)$ such that $\tilde{a}=a+y$ by the definition of $\calA_\delta(a^{\ast},b)$.
Since (\ref{hess1}) holds, 
it is sufficient to give lower bounds on the first and second terms in (\ref{hess1}) for $a$ in ${\cal A}_t$ in the sense of matrix inequalities. 
Since (\ref{alpha'}) holds
all we have to do is to show 
\begin{eqnarray}
\sigma''(f(a+y)-f(b))\nabla f(a+y)\nabla f(a+y)^{\top}
&\ge& -\frac{1}{2} l_0 \alpha I.
\Label{alpha''2}
\end{eqnarray}
Since $\sigma'$ is monotonically non-decreasing on $[-B,0]$, 
$\sigma''(z)$ is negative only if $z$ is positive.
Note that  $f(a)-f(b)\le0$ for any $a\in{\cal L}(a^{\ast},b)$ since $f(a^{\ast})-f(b)\le0$ and $f$ is convex.
Thus, we have
\begin{eqnarray}
&&\sigma''(f(a+y)-f(b))\nabla f(a+y)\nabla f(a+y)^{\top}\nonumber\\
&\ge&\left\{
\begin{array}{cll}
\sigma''(f(a+y)-f(b))L^2I&{\it if}&f(a+y)-f(b)>0\\
0&{\it if}&f(a+y)-f(b)\le0.
\end{array}
\right. \Label{low2}
\end{eqnarray}
When $f(a+y)-f(b)>0$ and $a\in{\cal L}(b,a^{\ast})$,
the following holds by the $L$-Lipshitzness of $f$:
\begin{eqnarray*}
f(a+y)-f(b)
&=&f(a)-f(b) -f(a)+f(a+y)\\
&\le&-f(a)+f(a+y)\\
&\le&L\delta.
\end{eqnarray*}
where we used again $f(a)-f(b)\le0$ for any $a\in\calL(a^{\ast},b)$.
Thus
\begin{eqnarray}
\sigma''(f(a+y)-f(b))
&=&\sigma''(f(a+y)-f(b)) - \sigma''(0)\nonumber\\
&\ge&-L_2(f(a+y)-f(b))\nonumber\\
&\ge&-LL_2\delta.
\Label{low3}
\end{eqnarray}
Combining (\ref{low2}) and (\ref{low3}),
we have 
\begin{eqnarray}
\sigma''(f(a+y)-f(b))\nabla f(a+y)\nabla f(a+y)^{\top}
&\ge&-L^3L_2\delta I.
\end{eqnarray}
Thus, when $\delta\le\frac{l_0 \alpha}{2L_0^3 L_2}$,
we obtain (\ref{alpha''2}).
\endproof

[{\bf Proof of Lemma \ref{unbiased}}]
We have
\begin{eqnarray}
\bbE[\hat{g}_t|a_t]
&=&\bbE_{u_t}[\bbE[\hat{g}_t|a_t, u_t]]\nonumber\\
&=&\bbE_{u_t}[ d\bbE[P_t(a_t+\nabla^2{\calR_t}(a_t)^{-\frac{1}{2}} u_t)\nabla^2{\calR_t}(a_t)^{\frac{1}{2}}u_t|a_t, u_t] ]\nonumber\\
&=&d\bbE[P_t(a_t+\nabla^2{\calR_t}(a_t)^{-\frac{1}{2}} u_t)\nabla^2{\calR_t}(a_t)^{\frac{1}{2}}u_t|a_t]\nonumber\\
%
&=&\nabla \bbE_{x\in \B}[ P_t(a_t+\nabla^2{\calR_t}(a_t)^{-\frac{1}{2}} x)|a_t]\Label{stokes}\\
&=&\nabla \hat{P}_t(a_t),\nonumber
\end{eqnarray}
where we used Stokes' theorem in (\ref{stokes}).
\endproof

[{\bf Proof of Lemma \ref{DB-bound}}]
We can divide the left hand side of (\ref{upper}) into three parts: 
\begin{eqnarray*}
&&\bbE\left[\sum_{t=1}^{T}(P_t(a_t) - P_t(a))\right]\\
&=& \bbE\left[\sum_{t=1}^{T}(\hat{P}_t(a_t) - \hat{P}_t(a))\right]
+\bbE\left[\sum_{t=1}^{T}(P_t(a_t) - \hat{P}_t(a_t))\right]
+\bbE\left[\sum_{t=1}^{T}(\hat{P}_t(a) - P_t(a))\right].
\end{eqnarray*}
Here, we bound the above three terms, respectively.
First, let us recall that $P_t$ is strongly convex on $\B(a_t, \frac{l_0\alpha}{2L_0^3L_2})\cap \calA$ due to Lemma \ref{convexity}.
By the definition of $\calR_t$ and the conditions for $\lambda$ and $\mu$, 
it holds that $a_t +\nabla^2{\calR_t}(a_t)^{-\frac{1}{2}} x \in \B(a_t, \frac{l_0\alpha}{2L_0^3L_2})\cap \calA$ for any $x\in\B$.
Thus, 
from the local convexity of $P_t$ in Lemma \ref{convexity} and Jensen's inequality,
$P_t(a_t) - \hat{P}_t(a_t)\le 0$ holds.
Next, 
from the smoothness of $P_t$ in Lemma \ref{smoothness},
we have 
\begin{eqnarray*}
\hat{P}_t(a) - P_t(a)
\le \frac{L_0\beta + B_2L^2}{2}\|\nabla^2{\calR_t}(a_t)^{-\frac{1}{2}}u\|^2
\le \frac{L_0\beta + B_2L^2}{2\lambda\eta t},
\end{eqnarray*}
where the first inequality follows from, for example, Lemma $7$ of \citet{hazan2014bandit}
and the second inequality follows from the definition of ${\calR_t}$.
Hence, we obtain 
\begin{eqnarray*}
\bbE\left[\sum_{t=1}^{T}(\hat{P}_t(a) - P_t(a))\right]
\le \frac{L_0\beta + B_2L^2}{\lambda\eta } \log T.
\end{eqnarray*}

Finally, we bound the first term of the upper bound of the regret.
We have the following inequalities:
\begin{eqnarray*}
&&\bbE[\hat{P}_t(a_t)-\hat{P}_t(a)]\\
&\le& \bbE\left[\nabla\hat{P}_t(a_t)^{\top}(a_t-a) - \frac{l_0\alpha}{4}\|a_t-a\|^2\right]\\
&=& \bbE\left[\hat{g}_t^{\top}(a_t-a) - \frac{l_0\alpha}{4}\|a_t-a\|^2 \right]\\
&=& \eta ^{-1}\bbE\left[(\nabla{\calR_t}(a_{t+1})-\nabla{\calR_t}(a_t))^{\top}(a-a_t) - \frac{l_0\alpha\eta}{4}\|a_t-a\|^2\right]\\
&=& \eta ^{-1}\bbE\left[D_{\calR_t}(a,a_t)+D_{\calR_t}(a_t,a_{t+1})-D_{\calR_t}(a,a_{t+1}) - \frac{l_0\alpha\eta}{4}\|a_t-a\|^2\right],
\end{eqnarray*}
where the first inequality follows from the local convexity of $\hat{P}_t$,
 the first equality is derived by Lemma \ref{unbiased}
 and the second equality holds due to the definition of $a_{t+1}$.
Summing up both sides,
\begin{eqnarray*}
&&
\bbE\left[\sum_{t=1}^T(\hat{P}_t(a_t)-\hat{P}_t(a)\right]\\
&\le& \eta ^{-1}\bbE\left[D_{{\calR_1}}(a,a_1) - D_{{\calR_T}}(a,a_{T+1}) + \sum_{t=1}^T D_{\calR_t}(a_t,a_{t+1})\right]\nonumber\\
&&+\eta ^{-1}\bbE\left[\sum_{t=2}^T \left(D_{\calR_t}(a,a_{t}) - D_{{\calR_{t-1}}}(a,a_{t}) - \frac{l_0\alpha\eta}{4}\|a_t-a\|^2\right)\right]\\
&\le& \eta ^{-1}\bbE\left[D_{{\calR_1}}(a,a_1) - D_{{\calR_T}}(a,a_T+1) + \sum_{t=1}^T D_{\calR_t}(a_t,a_{t+1})\right]\nonumber\\
&&+\eta ^{-1}\bbE\left[\sum_{t=2}^T \left(D_{\calR_t}(a,a_{t}) - D_{{\calR_{t-1}}}(a,a_{t}) - \frac{\lambda\eta}{2}\|a_t-a\|^2\right)\right]\\
&=& \eta ^{-1}\bbE\left[D_{{\calR_1}}(a,a_1) - D_{{\calR_T}}(a,a_T+1) + \sum_{t=1}^T D_{\calR_t}(a_t,a_{t+1})\right]\nonumber\\
&\le& \eta ^{-1}\bbE\left[D_{{\calR_1}}(a,a_1) + \sum_{t=1}^T D_{\calR_t}(a_t,a_{t+1})\right]\nonumber\\
&=& \eta ^{-1}\left({\calR}(a) - {\calR}(a_1) + \bbE\left[\sum_{t=1}^T D_{\calR_t^{\ast}}(\nabla{\calR_t}(a_t)-\eta \hat{g}_t,\nabla{\calR_t}(a_t))\right]\right)\nonumber
\label{bound3}
\end{eqnarray*}
where 
we used the positivity of the Bregman divergence in the third inequality
and  $\nabla{\calR}(a_1)=0$ because $a_1=\argmin{\calR}$ in the last equation.
Combining the above discussion,
we obtain (\ref{upper}).
\endproof

[{\bf Proof of Lemma \ref{div-bound}}]
Taylar's theorem guarantees the existence of $\delta_t\in(0,1)$ such that 
\begin{eqnarray}
D_{{\calR_t}^{\ast}}(\nabla\calR_t(a_t)-\eta \hat{g}_t,\nabla\calR_t(a_t))
&=&\eta ^2 \hat{g_t}^{\top} \nabla^2\calR_t^{\ast}(\nabla\calR_t(a_t)-\delta_t\eta \hat{g}_t)\hat{g}_t.
\Label{div}
\end{eqnarray}
Then using  the self-concordant property of $\calR^{\ast}_t$ (see e.g.  (F.2) of \citet{griva2009linear}), 
\begin{eqnarray}
\hat{g_t}^{\top} \nabla^2\calR_t^{\ast}(\nabla\calR_t(a_t)-\delta_t\eta \hat{g}_t)\hat{g}_t
&\le& \left(\frac{\|\hat{g}_t\|^{\ast}_{\nabla\calR_t(a_t)}}{1-
\delta_t\eta\|\hat{g}_t\|^{\ast}_{\nabla\calR_t(a_t)}}\right)^2.
\Label{norm1}
\end{eqnarray}
where $\|x\|^{\ast}_{y}=\sqrt{x^{\top} \nabla^2\calR^{\ast}_t(y)x}$. 
Here, we note that 
\begin{eqnarray}
\|\hat{g}_t\|^{\ast}_{\nabla\calR_t(x_t)}
&=&\sqrt{\hat{g_t}^{\top} \nabla^2\calR_t^{\ast}(\nabla\calR_t(x_t))\hat{g}_t} \nonumber\\
&=&\sqrt{\hat{g_t}^{\top} \nabla^2\calR_t(x_t)^{-1}\hat{g}_t }\nonumber \\
&\le&d\sqrt{ u_t^{\top}
\nabla^2\calR_t(x_t)^{\frac{1}{2}}
\nabla^2\calR_t(x_t)^{-1}
\nabla^2\calR_t(x_t)^{\frac{1}{2}} 
u_t} \nonumber\\
&=& d
\Label{norm2}
\end{eqnarray}
and thus, $\delta_t\eta\| \hat{g}_t\|^{\ast}_{\nabla{\calR}(a_t)}<\frac{1}{2}$ when $\eta \le \frac{1}{2d}$.
Consequently, (\ref{div}), (\ref{norm1}) and (\ref{norm2}) derives (\ref{eq:div-bound}).
\endproof

[{\bf Proof of Lemma \ref{DBFO}}]
We show the second inequality of (\ref{reg-ine}).
From the definition,
Since $l_0\le\alpha'_t$ and $\sigma(0)=\frac{1}{2}$ from Assumption \ref{ass3},
we have
\begin{eqnarray*}
Reg_T^{DB}
&=&\sup_{a\in{\cal A}}\bbE\left[\sum_{t=1}^{T}\big(\{\sigma(f(a_t) - f(a)) - \sigma(0)\} + \{\sigma(f(a'_t) - f(a)) - \sigma(0)\}\big)\right]\\
&\ge&l_0\sup_{a\in{\cal A}}\bbE\left[\sum_{t=1}^{T}\big(\{f(a_t) - f(a) \} + \{f(a'_t) - f(a) \}\big)\right]\\
&=&l_0Reg_T^{FO}.
\end{eqnarray*}
The first inequality of (\ref{reg-ine}) can be proven in a similar manner.
\endproof

\end{document}